\documentclass[accepted]{uai2023} 

\usepackage[american]{babel}

\usepackage{natbib} 
\usepackage{mathtools} 
\usepackage{booktabs} 
\usepackage{tikz} 

\usepackage{multirow}
\usepackage{url}            
\usepackage{booktabs}       
\usepackage{amsfonts}       
\usepackage{nicefrac}       
\usepackage{microtype}      
\usepackage{hyperref}
\usepackage{graphicx}
\usepackage{amssymb,amsmath,amsthm}
\usepackage{epstopdf}
\usepackage{algorithm,algcompatible}
\usepackage{mathtools}
\usepackage{wrapfig}
\usepackage{soul}
\usepackage{tikz}
\usepackage{xspace}
\usepackage{subcaption}

\usepackage[inline]{enumitem}

\definecolor{Blue}{rgb}{0.2,0.2,0.6}
\hypersetup{
  colorlinks=true,
  citecolor=Blue,
  linkcolor=Blue,
  urlcolor=Blue,
}

\usetikzlibrary{fit,positioning,arrows,automata}

\mathtoolsset{showonlyrefs}

\newtheorem{definition}{Definition}
\newtheorem{proposition}{Proposition}
\newtheorem{lemma}{Lemma}

\newcommand{\mR}{{\mathbb R}}

\newcommand{\cD}{{\mathcal D}}

\newcommand{\cO}{{\mathcal O}}
\newcommand{\cP}{{\mathcal P}}

\newcommand{\cT}{{\mathcal T}}

\newcommand{\cW}{{\mathcal W}}
\newcommand{\cX}{{\mathcal X}}

\def\OT{{\textup{OT}}}

\def\<{{\langle}}
\def\>{{\rangle}}

\def\R{{\mathbb{R}}}
\def\d{{\text{d}}}
\newcommand{\fan}[1]{{\color{black}{#1}}}

\DeclareMathOperator*{\argmin}{arg\,min}

\everypar{\looseness=-1}

\title{Generating Synthetic Datasets by \\ Interpolating along Generalized Geodesics}

\stepcounter{footnote} 
\author[1]{\href{mailto:<jiaojiaofan@gatech.edu>?Subject=Your UAI 2023 paper}{Jiaojiao Fan\thanks{Work done partly during an internship at Microsoft Research.}}{}} 
\author[2]{David Alvarez-Melis}
\affil[1]{%
    Georgia Tech\\
    Atlanta, Georgia, USA
}
\affil[2]{%
    Microsoft Research \& Harvard University\\
    Cambridge, Massachusetts, USA
}

\begin{document}
\maketitle

\begin{abstract}
  Data for pretraining machine learning models often consists of collections of heterogeneous datasets. Although training on their union is reasonable in agnostic settings, it might be suboptimal when the target domain ---where the model will ultimately be used--- is known in advance. In that case, one would ideally pretrain only on the dataset(s) most similar to the target one. Instead of limiting this choice to those datasets already present in the pretraining collection, here we explore extending this search to all datasets that can be synthesized as `combinations' of them. We define such combinations as multi-dataset interpolations, formalized through the notion of generalized geodesics from optimal transport (OT) theory. We compute these geodesics using a recent notion of distance between labeled datasets, and derive alternative interpolation schemes based on it: using either barycentric projections or optimal transport maps, the latter computed using recent neural OT methods. These methods are scalable, efficient, and ---notably--- can be used to interpolate even between datasets with distinct and unrelated label sets. Through various experiments in transfer learning in computer vision, we demonstrate this is a promising new approach for targeted on-demand dataset synthesis. 
\end{abstract}

\section{Introduction}

Recent progress in machine learning has been characterized by the rapid adoption of large pretrained models as a fundamental building block \citep{brown2020language}. These models are typically pretrained on large amounts of general-purpose data and then adapted (e.g., \textit{fine-tuned}) to a specific task of interest. Such pretraining datasets usually draw from multiple heterogeneous data sources, e.g., arising from different domains or sources. Traditionally, all available datasets are used in their entirety during pretraining, for example by pooling them together into a single dataset (when they all share the same label sets) or by training in all of them sequentially one by one. These strategies, however, come with important disadvantages. Training on the union of multiple datasets might be prohibitive or too time-consuming, and it might even be detrimental --- indeed, there is a growing line of research showing that removing pretraining data sometimes improves transfer performance \citep{jain2022data}. On the other hand, sequential learning (i.e., consuming datasets one by one) is infamously prone to \textit{catastrophic forgetting} \citep{mccloskey1989catastrophic, kirkpatrick2017overcoming}: the information from earlier datasets gradually vanishing as the model is trained on new datasets. The pitfalls of both of these approaches suggest training instead on a \textit{subset} of the available pretraining datasets, but how to choose that subset is unclear. However, when the target dataset on which the model is to be used is known in advance, the answer is much easier: intuitively, one would train only of those relevant to the target one: e.g., those most similar to it. Indeed, recent work has shown that selecting pretraining datasets based on the distance to the target is a successful strategy \citep{alvarez2020geometric, gao2021information}. However, such methods are limited to selecting (only) among individual datasets already present in the collection.

In this work, we propose a novel approach to \textit{generate} synthetic pretraining datasets as combinations of existing ones. In particular, this method searches among all possible continuous combinations of the available datasets and thus is not limited to selecting specifically one of them. When given access to the target dataset of interest, we seek among all such combinations the one closest (in terms of a metric between datasets) to the target. By characterizing datasets as sampled from an underlying probability distribution, this problem can be understood as a generalization (from Euclidean to probability space) of the problem of finding among the convex hull of a set of reference points, that which is closest to a query point. While this problem has a simple closed-form solution in Euclidean space (via an orthogonal projection), solving it in probability space is ---as we shall see here--- significantly more challenging.

We tackle this problem from the perspective of interpolation. Formally, we model the combination of datasets as an interpolation between their distributions, formalized through the notion of geodesics in probability space endowed with the Wasserstein metric \citep{ambrosio2008gradient, villani2008optimal}. In particular, we rely on \textit{generalized geodesics} \citep{craig2016exponential, ambrosio2008gradient}: constant-speed curves connecting a pair (or more) distributions parametrized with respect to a `base' distribution, whose role is played by the target dataset in our setting. Computing such geodesics requires access to either an optimal transport coupling or a map between the base distribution and every other reference distribution. The former can be computed very efficiently with off-the-shelf OT solvers, but are limited to generating only as many samples as the problem is originally solved on. In contrast, OT maps allow for on-demand out-of-sample mapping and can be estimated using recent advances in neural OT methods \citep{fan2020scalable, korotin2022neural, makkuva2020optimal}. However, most existing OT methods assume unlabeled (feature-only) distributions, but our goal here is to interpolate between classification (i.e., labeled) datasets. Therefore, we leverage a recent generalization of OT for labeled datasets to compute couplings \citep{alvarez2020geometric} and adapt and generalize neural OT methods to the labeled setting to estimate OT maps.

In summary, the contributions of this paper are:
\begin{enumerate*}[label=(\roman*)]
  \item a novel approach to generate new synthetic classification datasets from existing ones by using geodesic interpolations, applicable even if they have disjoint label sets,
  \item two efficient methods to solve OT between labeled datasets,
  which might be of independent interest,
  \item empirical validation of the method in various transfer learning settings.
\end{enumerate*}

\section{Related work}


\paragraph{Mixup and related In-Domain Interpolation} Generating training data through convex combinations was popularized by \textit{mixup} \citep{zhang2018mixup}: a simple data augmentation technique that interpolates features and labels between pairs of points. This and other works based on it \citep{zhang2021how, chuang2021fair,yao2021meta} use mixup to improve in-domain model robustness~\citep{zhu2023interpolation} and generalization by increasing the in-distribution diversity of the training data. Although sharing some intuitive principles with mixup, our method interpolates entire datasets ---rather than individual datapoints--- with the goal of improving across-distribution diversity and therefore out-of-domain generalization.

\paragraph{Dataset synthesis in machine learning} Generating data beyond what is provided as a training dataset is a crucial component of machine learning in practice. Basic transformations such as rotations, cropping, and pixel transformations can be found in most state-of-the-art computer vision models. More recently, Generative Adversarial Nets (GAN) have been used to generate synthetic data in various contexts~\citep{bowles2018GAN, yoon2019pate-gan},
a technique that has proven particularly successful in the medical imaging domain \citep{sandfort2019data}. Since GANs are trained to replicate the dataset on which they are trained, these approaches are typically confined to generating in-distribution diversity and typically operate on features only.

\paragraph{Discrete OT, Neural OT, Gradient Flows}
Barycentric projection~\citep{ambrosio2008gradient,perrot2016mapping} is a simple and effective method to approximate an optimal transport map with discrete regularized OT.
On the other hand, there has been remarkable recent progress in methods to estimate OT maps in Euclidean space using neural networks~\citep{makkuva2020optimal,fan2021scalable,rout2022generative}, which have been successfully used for image generation~\citep{rout2022generative}, style transfer~\citep{korotin2022neural}, among other applications. However, the estimation of an optimal map between (labeled) datasets has so far received much less attention. Some conditional Monge map solvers~\citep{bunne2022supervised} 
utilize the label information in a semi-supervised manner, where they assume the label-to-label correspondence between two distributions is known. Our method differs from theirs in that we do not require a pre-specified label-to-label mapping, but instead estimate it from data. Geodesics and interpolation in general metric spaces have been studied extensively in the optimal transport and metric geometry literature \citep{mccann1997convexity, agueh2011barycenters, ambrosio2008gradient, santambrogio2015optimal, villani2008optimal, craig2016exponential}, albeit mostly in a theoretical setting. Gradient flows \citep{santambrogio2015optimal}, increasingly popular in machine learning to model existing processes
\citep{bunne2022proximal, mokrov2021large-scale, fan2022variational, hua2023dynamic} or solving optimization problems over datasets \citep{alvarez2021dataset}, provide an alternative approach for interpolation between distributions but are computationally expensive.




\section{Background}
\subsection{Distribution interpolation with OT}
Consider $\cP(\cX)$ the space of probability distributions with finite second moments over some Euclidean space $\cX$. Given $\mu,\nu\in \cP(\cX)$, the Monge formulation optimal transport problem seeks a map $T:\cX\rightarrow\cX$ that transforms $\mu$ into $\nu$ at a minimal cost. Formally, the objective of this problem is
$\min_{T: T\sharp \mu = \nu} \int_{\mR^d} \|x-T(x)\|_2^2 \d\mu(x),$
where the minimization is over all the maps that pushforward distribution $\mu$ into distribution $\nu$. While a solution to this problem might not exist, a relaxation due to Kantorovich is guaranteed to have one. This modified version yields the Wasserstein-2 distance:
$ W_2^2(\mu,\nu) = \min_{\pi \in \Pi(\mu, \nu)} \int_{\mR^d} \|x-x'\|_2^2 \d\pi(x, x'),$
where now the constraint set $\Pi(\mu,\nu)= \{ \pi \in \mathcal{P}(\mathcal{X}^2) \mid P_{0\sharp}\pi = \mu, P_{1\sharp\pi}=\nu \}$ contains all couplings with marginals $\mu$ and $\nu$. The optimal such coupling is known as the OT plan. A celebrated result by \citet{brenier1991polar} states that whenever $P$ has density with respect to the Lebesgue measure, the optimal $T^*$ exists and is unique. In that case, the Kantorovich and Monge formulations coincide and their solutions are linked by $\pi^* = (\text{Id}, T^*)_\sharp \mu$ where $\rm Id$ is the identity map. The Wasserstein-2 distance enjoys many desirable geometrical properties compared to other distances for distributions \citep{ambrosio2008gradient}. One such property is the characterization of geodesics in probability space \citep{agueh2011barycenters,santambrogio2015optimal}. When $\cP(\cX)$ is equipped with metric $W_2$, the unique minimal geodesic between any two distributions $\mu_0$ and $\mu_1$ is fully determined by $\pi$, the optimal transport plan between them, through the relation:
\begin{align}\label{eq:displacement}
  \rho_t^D: = ((1-t)x + t y )\sharp \pi(x,y) , \quad t \in [0,1],
\end{align}
known as \emph{displacement interpolation}. If the Monge map from $\mu_0$ to $\mu_1$ exists, the geodesic can also be written as
\begin{align}\label{eq:mccan}
  \rho^M_t: = ((1-t) {\rm Id} + t T^* )\sharp \mu_0 , \quad t \in [0,1],
\end{align}
and is known as \emph{McCann's interpolation}~\citep{mccann1997convexity}. It is easy to see that $\rho^M_0=\mu_0$ and $\rho^M_1=\mu_1$.

Such interpolations are only defined between two distributions.
When there are $m \ge 2$ marginal distributions $\{\mu_1, \ldots, \mu_m\}$, the \emph{Wasserstein barycenter}
\begin{align}
  \rho^B_a: = \argmin_\rho \sum_{i=1}^m a_i W_2^2(\rho, \mu_i) , \quad a \in \Delta_{m-1} \subset \mR_{\ge 0}^m
\end{align}
generalizes McCann's interpolation~\citep{agueh2011barycenters}. Intuitively, the interpolation parameters $a = [a_1,\dots, a_m]$ determine the `mixture proportions' of each dataset in the combination, akin to the weights in a convex combination of points in Euclidean space. In particular, when $a$ is a one-hot vector with $a_i=1$, then $\rho^B_a = \mu_i$, i.e., the barycenter is simply the $i$-th distribution. Barycenters have attracted significant attention in machine learning recently~\citep{srivastava2018scalable,korotin2021continuous}, but they remain challenging to compute in high dimension~\citep{fan2020scalable,korotin2022wasserstein}.

Another limitation of these interpolation notions is the non-convexity of $W_2^2$ along them. In Euclidean space, given three points $x_1,x_2,y \in \mR^d$, the function $t \mapsto \|x_t-y\|_2^2$, where $x_t$ is the interpolation $x_t = (1-t) x_1 + t x_2 $, is convex.
In contrast, in Wasserstein space, neither the function $t \mapsto W_2^2(\rho^M_t , \nu)$ or $a \mapsto W_2^2(\rho^B_a , \nu)$ are guaranteed to be convex~\citep[\S4.4]{santambrogio2017euclidean}. This complicates their theoretical analysis, such as in gradient flows.
To circumvent this issue, \citet{ambrosio2008gradient} introduced the \emph{generalized geodesic} of $\{\mu_1,\ldots,\mu_m\}$ with base $\nu\in\mathcal{P}(\mathcal{X})$:
\begin{align}
  \rho^G_a := \left(\sum_{i=1}^m a_i T^*_i \right)\sharp \nu  , \quad a \in \Delta_{m-1},
\end{align}
where $T^*_i $ is the optimal map from $\nu$ to $\mu_i$.

\begin{lemma}\label{lem:convex_w2}
  The functional $\mu \mapsto W_2^2( \mu, \nu)$ is convex along the generalized geodesics, and
  $
    W_2^2(\rho^G_a, \nu ) \le  \sum_{i=1}^m a_i W_2^2(\mu_i, \nu) .
  $
\end{lemma}

Thus, unlike the barycenter, the generalized geodesic does yield a notion of convexity satisfied by the Wasserstein distance and is easier to compute. 
The proof of Lemma \ref{lem:convex_w2} is postponed to \S A.
For these reasons, we adopt this notion of interpolation for our approach. It remains to discuss how to use it on (labeled) datasets.

\subsection{Dataset distance}
Consider a dataset $\cD_P = \{ z^{(i)}\}_{i=1}^N = \{ x^{(i)}, y^{(i)}\}_{i=1}^N \overset{i.i.d.}{\sim} P(x,y)$. The Optimal Transport Dataset Distance (OTDD)~\citep{alvarez2020geometric} measures its distance to another dataset $\cD_Q$ as:
\begin{align}\label{eq:otdd}
   & d^2_{\OT} (\cD_P,\cD_Q ) = \nonumber                                                                    \\
   & \min_{\pi \in \Pi (P,Q)} \int \left( \|x-x'\|_2^2 + W_2^2(\alpha_y, \alpha_{y'}) \right) \d\pi (z,z' ),
\end{align}
which defines a proper metric between datasets. Here, $\alpha_y, \alpha_{y'}$ are class-conditional measures corresponding to $P(x|y)$ and $Q(x|y')$. This distance is strongly correlated with transfer learning performance, i.e., the accuracy achieved when training a model on
$\cD_P$ and then fine-tuning and evaluating on $\cD_P$. Therefore, it can be used to select pretraining datasets for a given target domain. Henceforth we abuse the notation $P$ to represent both a dataset and its underlying distribution for simplicity. To avoid confusion, we use $\nu$ and $\mu$ to represent distributions in the feature space (typically $\mathbb{R}^d$) and use $P,Q$ to represent distributions in the product space of features and labels.


\section{Dataset interpolation along generalized geodesic}
Our method consists of two steps: estimating optimal transport maps between the target dataset and all training datasets (\S\ref{sec:map}), and using them to generate a convex combination of these datasets by interpolating along generalized geodesics (\S\ref{sec:comb}). 
\fan{The OT map estimation is in feature space or original space depending on the dataset's dimension.} 
For some downstream applications, we will additionally project the target dataset into the `convex hull' of the training datasets (\S\ref{sec:proj}).


\subsection{Estimating optimal maps between labeled datasets}\label{sec:map}
The OTDD is a special case of Wasserstein distance, so it is natural to consider the alternative Monge (map-based) formulation to \eqref{eq:otdd}.
We propose two methods to approximate the OTDD map, one using the entropy-regularized OT and another one based on neural OT.

\paragraph{OTDD barycentric projection.}
Barycentric projections~\citep{ambrosio2008gradient,pooladian2021entropic} can be efficiently computed for entropic regularized OT using the Sinkhorn algorithm~\citep{sinkhorn1967diagonal}.
Assume that we have empirical distributions $\nu = \sum_{i=1}^{N_\nu} \frac{1}{N_\nu} \delta_{x_\nu^{(i)}} $ and $ \mu = \sum_{i=1}^{N_\mu} \frac{1}{N_\mu} \delta_{x_\mu^{(i)}} $, where $\delta_x$ is the Dirac function at $x$.
Denote all the samples compactly as matrices: $X_\nu = \left(x_\nu^{(1)}, \ldots, x_\nu^{(N_\nu)} \right) \in \mR^{N_\nu \times d}, X_\mu = \left(x_\mu^{(1)}, \ldots, x_\mu^{(N_\mu)} \right) \in \mR^{N_\mu \times d} $.
After solving the optimal coupling
$\pi^*: =
  \min_{\pi \in \Pi (\nu, \mu)}
  \sum_{i,j} \|x_\nu^{(i)}-x_\mu^{(j)}\|^2 \pi(i,j)
$,
the barycentric projection can be expressed as
$T_B (X_\nu) = N_\nu  \pi^* X_\mu.$
We extend this method to two datasets $Z_Q = \{X_Q, Y_Q \}, Z_P = \{X_P, Y_P \}$, where we have additional one-hot label data \fan{$Y_Q = (y_Q^{(1)}, \ldots, y_Q^{(N_Q)} ) \in \{0,1\}^{N_Q \times C_Q},
Y_P = (y_P^{(1)}, \ldots, y_{P}^{(N_P)}) \in
\{0,1\}^{N_P \times C_P }
$.} $C_Q$ and $C_P$ are the number of classes in dataset $Q$ and $P$.
We solve the optimal coupling $\pi^* \in \mR_{\ge 0}^{N_P \times N_Q}$ for OTDD~\eqref{eq:otdd} following the regularized scheme in ~\citet{alvarez2020geometric}.
The barycentric projection can then be written as:
\begin{align}\label{eq:bary_map}
  \cT_B (Z_Q) = [N_Q  \pi^* X_P , N_Q  \pi^* Y_P].
\end{align}
The visualization of barycentric projected data appears in Figure \ref{fig:bary_projection}.
\begin{figure}[t]
  \centering
  \begin{subfigure}{0.51\textwidth}
    \centering
    \includegraphics[width=1\linewidth]{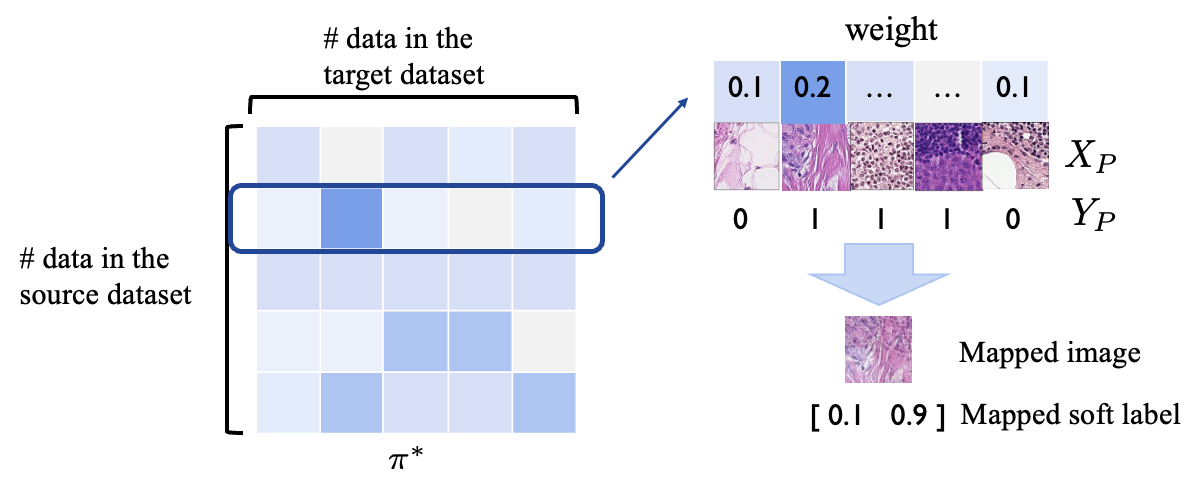}
  \end{subfigure}
  \caption{Visualization of OTDD barycentric projection on binary PCAM  dataset. \fan{
  We first solve the optimal coupling $\pi^* \in [0,1]^{N_Q \times N_P}$ for the problem \eqref{eq:otdd} using entropy regularization. Next, we map the $i$-th datapoint in the source dataset to a pair consisting of a weighted image and a weighted soft label. The weight vector, extracted from the $i$-th row of the coupling $\pi^*$, is then normalized to sum to 1. As a result, the mapped image (or soft label) is obtained as a convex combination of all the images (or one-hot labels) in the target dataset.
  }
  }
  \label{fig:bary_projection}
\end{figure}
However, this approach has two important limitations: it can not naturally map out-of-sample data and it does not scale well to large datasets (due to the quadratic dependency on sample size). \fan{To relieve the scaling issue, we will use batchified version of OTDD barycentric projection in this paper (see complexity discussion in \S\ref{sec:conclude}).}

\paragraph{OTDD neural map.} Inspired by recent approaches to estimate Monge maps using neural networks \citep{rout2022generative,fan2021scalable}, we design a similar framework for the OTDD setting. \citet{fan2021scalable} approach the Monge OT problem with general cost functions by solving its max-min dual problem $$\sup_f\inf_T  \int \left[ c( x ,T(x))-f(T(x)  )\right] \d\nu(x)  + \int f(x') \d\mu(x').$$
We extend this method to the distributions involving labels by introducing an additional classifier in the map. Given two datasets $P,Q$, we parameterize the map $\cT_N: \mR^d \times [0,1]^{C_Q} \rightarrow  \mR^d \times [0,1]^{C_P}$ as
\begin{align}\label{eq:map}
  \cT_N(z) = \cT_N(x,y)  =[\bar x ; \bar y ]= [G(z) ; \ell(G(z))],
\end{align}
where $G:
  \mR^d \times [0,1]^{C_Q} \rightarrow \mR^d $ is the pushforward feature map, and the $\ell:  \mR^d \rightarrow [0,1]^{C_P} $ is a frozen classifier that is pre-trained on the dataset $P$.
Notice that, with the cost $c(z, \cT(z)) =  \| x- G( z)\|_2^2 + W_2^2(\alpha_{y}, \alpha_{\bar y} ) $, the Monge formulation of OTDD \eqref{eq:otdd} reads $\inf_{T\sharp Q = P }  \int \| x- G( z)\|_2^2 + W_2^2(\alpha_{y}, \alpha_{\bar y} ) \d Q(z).$
We therefore propose to solve the max-min dual problem
\begin{align}\label{eq:max-min}
  \sup_f \inf_G  \int \left[ \| x- G( z)\|_2^2 + W_2^2(\alpha_{y}, \alpha_{\bar y} )\right] \d Q(z) \nonumber \\
  - \int f(\bar x , \bar y ) \d Q(z)
  + \int f(x' ,y') \d P(z').
\end{align}

\begin{figure}[t]
  \centering
  \begin{subfigure}{0.5\textwidth}
    \centering
    \includegraphics[width=1\linewidth]{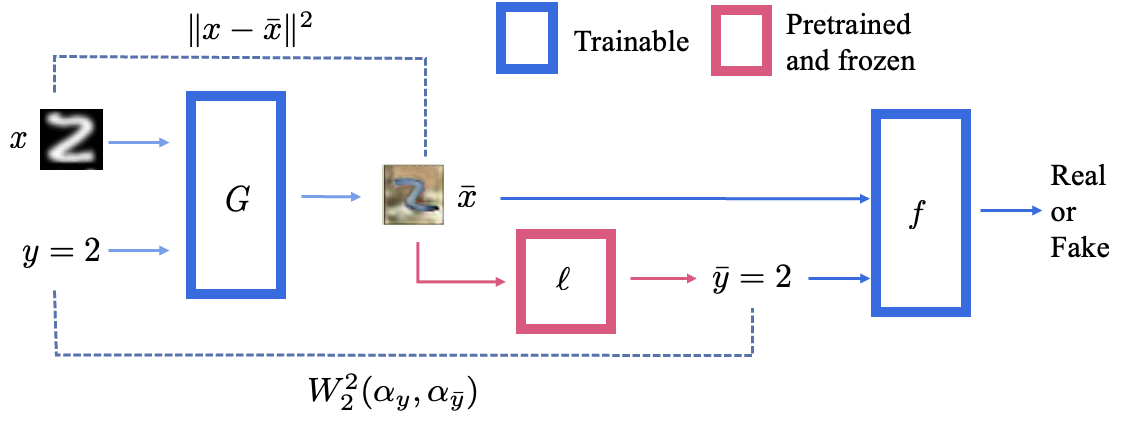}
  \end{subfigure}
  \caption{Training paradigm for learning the OTDD neural map betweem two datasets (distributions), parametrized via a pushforward feature map $G$ and a labeling function $\ell$, using 
  a dual potential
  $f$.
  }
  \label{fig:neural_map}
\end{figure}

Implementation details are provided in \S B.
Compared to previous conditional Monge map solvers~\citep{bunne2022supervised,asadulaev2022neural}, the two methods proposed here: (i) do not assume class overlap across datasets, allowing for maps between datasets with different label sets; (ii) are invariant to class permutation and re-labeling; (iii) do not force one-to-one class alignments, e.g., samples can be mapped across similar classes.

\subsection{Convex combination in dataset space}\label{sec:comb}
\begin{figure}[ht!]
  \centering
  \begin{subfigure}{0.51\textwidth}
    \centering
    \includegraphics[width=1\linewidth]{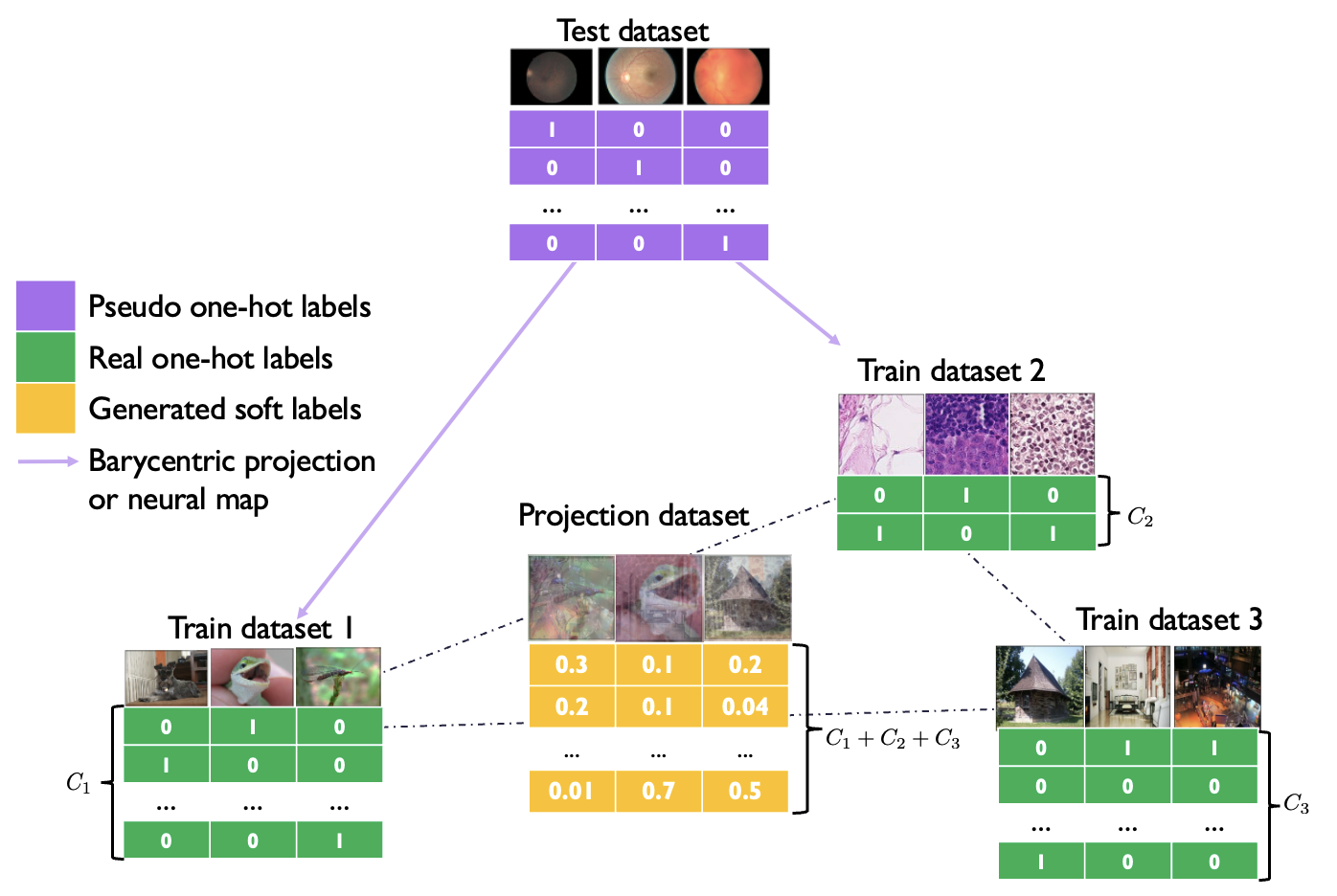}
  \end{subfigure}
  \caption{
    In few-shot settings, we use pseudo-labels for the test dataset, generated e.g.~via kNN using the few-shot examples. 
    If more labeled data from the test dataset is available, we use it instead of the pseudo-labels. \fan{The projection dataset has the same number of samples as the test dataset.}
  }
  \label{fig:convex_comb}
\end{figure}

Computing generalized geodesics requires constructing convex combinations of datapoints from different datasets. Given a weight vector $a \in \Delta_{m-1}$, features can be naturally combined as $x_a = \sum_{i=1}^m a_i x_i $. But combining labels is not as simple because: (i) we allow for datasets with a different number of labels, so adding them directly is not possible; (ii) we do not assume different datasets have the same label sets, e.g. MNIST (digits) vs CIFAR10 (objects). Our solution is to represent all labels in the same dimensional space by padding them with zeros in all entries corresponding to other datasets. As an example, consider three datasets with $2,3$, and $4$ classes respectively. Given a label vector $y_1\in \R^3$ for the first one, we embed it into $\R^9$ as
$\tilde{y}_1 = [y_1; \mathbf{0}_3; \mathbf{0}_4]^\top.$
Defining $\tilde{y}_2, \tilde{y}_3$ analogously, we compute their combination  as
$y_a = a_1\tilde{y}_1 + a_2\tilde{y}_2 + a_3\tilde{y}_3$.
This representation is lossless and preserves the distinction of labels across datasets.
The visualization of our convex combination is in Figure \ref{fig:convex_comb}.


\begin{figure*}[ht!]
  \centering
  \begin{subfigure}{0.25\textwidth}
    \centering
    \includegraphics[width=1\linewidth]{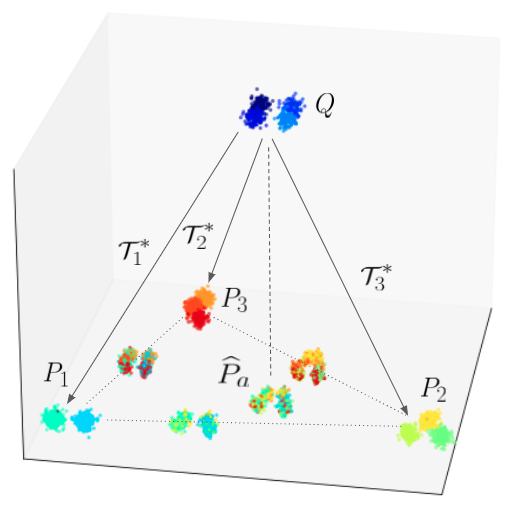}
    \caption{}
  \end{subfigure}
  \begin{subfigure}{0.65\textwidth}
    \centering
    \includegraphics[width=1\linewidth]{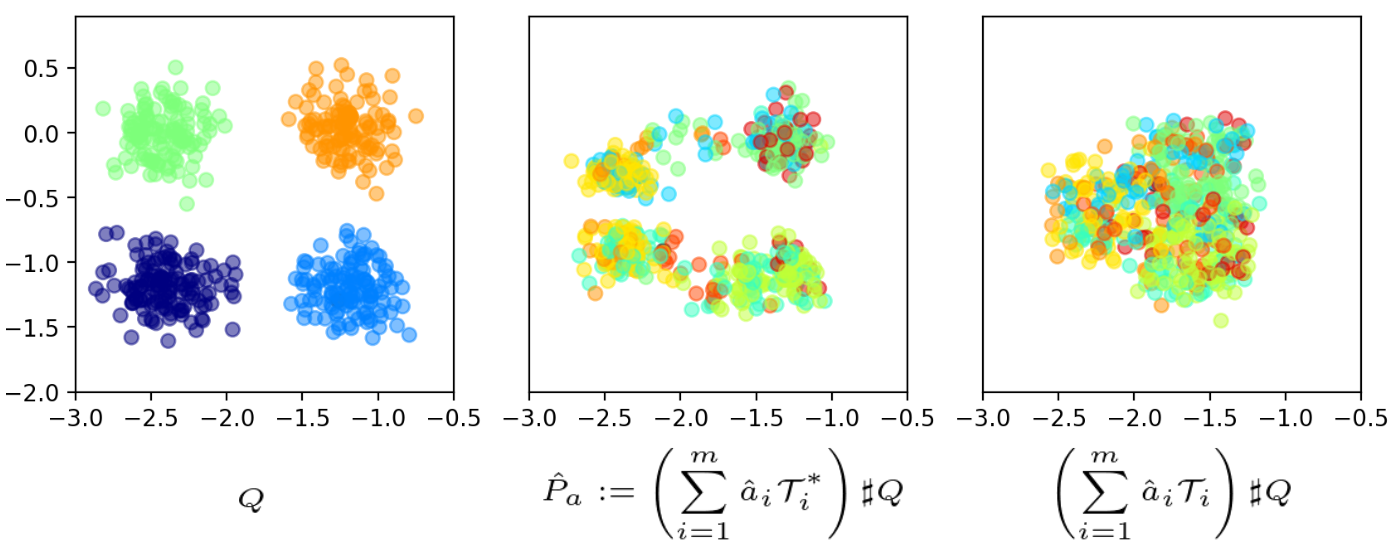}
    \caption{}
  \end{subfigure}
  \caption{\textbf{Visualization and comparison of dataset interpolation methods.} (a) The reference dataset $Q$ (with color-coded classes) is projected onto the generalized geodesic of the training datasets $P_i$, resulting in $\widehat{P}_a$. (b) 2D visualizations of (left-to-right): dataset $Q$, the `optimal' interpolated dataset $ \hat{P}_a := \left(\sum_{i=1}^m \hat a_i \cT^*_i \right)\sharp Q$ using the true OTDD maps $\cT_i^*$ , and a naively interpolated dataset $\left(\sum_{i=1}^m \hat a_i \cT_i \right)\sharp Q $ using randomly generated maps $\cT_i$.
  }
  \label{fig:proj_2d}
\end{figure*}
\subsection{Projection onto generalized geodesic of datasets}\label{sec:proj}

We now put together the components in Sec \ref{sec:map} and \ref{sec:comb} to construct generalized geodesics between datasets in two steps. First, we compute OTDD maps $\cT_i^*$ between $Q$ and all other datasets $P_i, i=1,\ldots, m$ using the discrete or neural OT approaches. Then, for any interpolation vector  $a \in \Delta_{m-1}$ we identify a dataset along the generalized geodesic via
\begin{align}
  P_a := \left(\sum_{i=1}^m a_i \cT^*_i \right)\sharp Q.
\end{align}
By using the convex combination method in \S\ref{sec:comb} for labeled data, we can efficiently sample from $P_a$. 

Next, we find the dataset $P^*_a$ that minimizes the distance between $P_a$ and $Q$, i.e. the projection of $Q$ onto the generalized geodesic. We first approach this problem from a Euclidean viewpoint.
Suppose there are several distributions $\{\mu_i\}_{i=1}^m$ and an additional distribution $\nu$ on Euclidean space $\mR^d$. Lemma \ref{lem:convex_w2} guarantees
there exists a unique parameter $a^*$ that minimizes $W_2^2(\rho_{a}^G, \nu)$. However, finding $a^*$ is far from trivial because there is no closed-form formula of the map $a \mapsto W_2^2(\rho_{a}^G, \nu)$ and it can be expensive to calculate $W_2^2(\rho_{a}^G, \nu)$ for all possible $a$. To solve this problem, we resort to another transport distance: the (2,$\nu$)-transport metric.

\begin{definition}[\citet{craig2016exponential}]
  Given distributions $\mu_i, \mu_j$, the (2,$\nu$)-transport metric between them is given by $$W_{2,\nu}(\mu_i,\mu_j) := \left( \int \|T_i^*(x) - T_j^*(x) \|_2^2 \d \nu (x) \right)^{1/2},$$ where $T_i^*$ is the optimal map from $\nu$ to $\mu_i$.
\end{definition}
When $\nu$ has a density with respect to Lebesgue measure $W_{2,\nu}$ is a valid metric~\citep[Prop. 1.15]{craig2016exponential}. Moreover, we can derive a closed-form formula for the map $a \mapsto W_{2, \nu }^2(\rho_{a}^G, \nu)$.
\begin{proposition}\label{prop:eq}
  $W_{2, \nu }^2(\rho^G_a, \nu )
    =  \sum_{i=1}^m a_i W_{2,\nu }^2(\mu_i, \nu ) - \frac{1}{2} \sum_{i \neq j} a_i a_j W_{2,\nu }^2(\mu_i, \mu_j ).
  $
\end{proposition}
This equation
implies that given distributions $\{\mu_i\}, \nu$ in Euclidean space, we can trivially solve the optimal $a^*$ that minimizes $W_{2, \nu }^2(\rho^G_a, \nu ) $ by a quadratic programming solver\footnote{We use the implementation \url{https://github.com/stephane-caron/qpsolvers}}. The proof (\S A
) relies on Brenier's theorem.
Inspired by this, we also define a transport metric for datasets:
\begin{definition}\label{def:Q_ds_distance}
  The squared (2,$Q$)-dataset distance is  $$\cW^2_{2,Q}(P_i, P_j) := \int \left( \| x_i -  x_j \|_2^2 +
    W_2^2(\alpha_{y_i}, \alpha_{y_j})
    \right) \d Q(z), $$ where $ [ x_i;  y_i ] =\cT_i^*(z)$ and $\cT_i^*$ is the OTDD map from $Q$ to $P_i$.
\end{definition}
Denote $\cP_{2,Q} (\cX \times \cP(\cX) ) $ as the set of all probability measures $P$ that satisfy $  d_\OT (P,Q) < \infty $ and the OTDD map from $Q$ to $P $ exists. The following result shows that (2,$Q$)-dataset distance is a proper distance. The proof is again deferred to \S A.
\begin{proposition}\label{prop:metric}
  $\cW_{2,Q}$ is a valid metric on $\cP_{2,Q}(\cX \times \cP(\cX))$.
\end{proposition}
Unfortunately, in this case $\cW^2_{2,Q}(P_a, Q)$ does not have an analytic form like before because Brenier's theorem may not hold for a general transport cost problem.
However, we still borrow this idea and define an approximated projection $\widehat{P}_a$ as the minimizer of function
\begin{align}\label{eq:ds_eq}
   & \cW^2(P_a, Q): =  \nonumber                                                                       \\
   & \sum_{i=1}^m a_i \cW^2_{2,Q}(P_i, Q) - \frac{1}{2} \sum_{i \neq j} a_i a_j \cW^2_{2,Q}(P_i, P_j),
\end{align}
which is an analog of Proposition \ref{prop:eq}. 
\fan{Since ${P}_a$ is defined by its interpolation weight $a$, solving $\widehat{P}_a$ is equivalent to finding a weight 
\begin{align}\label{eq:hat_a}
    \hat a = \argmin_{a \in \Delta_{m-1}} \cW^2(P_a, Q),
\end{align}
which is a simple quadratic programming problem.}
Unlike the Wasserstein distance, $\cW^2_{2,Q}(\cdot, \cdot)$ is easier to compute because it does not involve optimization, so it is relatively cheap to find the minimizer of $\cW^2(P_a, Q)$.
Experimentally, we observe that
$W_{2, Q }^2(P_a, Q )$ is predictive of model transferability across tasks. Figure \ref{fig:proj_2d}(a) illustrates this projection on toy 3D datasets, color-coded by class.



\section{Experiments}

\subsection{Learning OTDD maps}
In this section, we visualize the quality of the learnt OTDD maps on both  synthetic and realistic datasets.
\paragraph{Synthetic datasets}





Figure \ref{fig:proj_2d} (b) illustrates the role of the optimal map in estimating the projection of a dataset into the generalized geodesic hull of three others. Using maps $\cT_i^*$ estimated via barycentric projection \eqref{eq:bary_map} results in better preservation of the four-mode class structure, whereas using non-optimal maps $\cT_i$ based on random couplings (as the usual \textit{mixup} does) destroys the class structure.

\paragraph{*NIST datasets}
\begin{figure}[h!]
  \centering
  \begin{subfigure}{0.48\textwidth}
    \centering
    \includegraphics[width=1\linewidth]{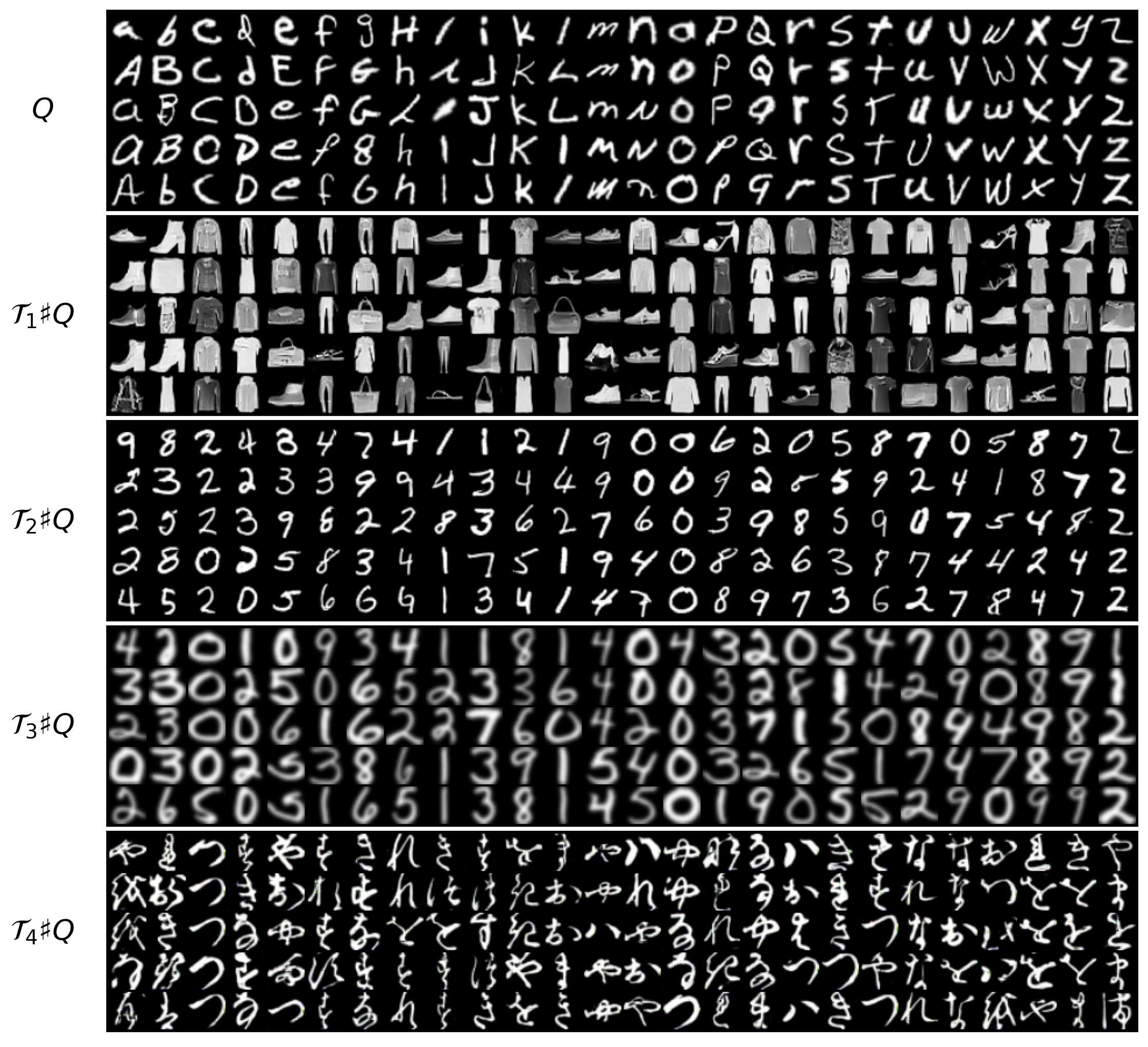}
  \end{subfigure}
  \caption{Datasets generated by pushing forward $Q$ (the EMNIST dataset) towards Fashion-MNIST, MNIST, USPS, KMNIST, using OTDD maps $\mathcal{T}_i$, obtained using the neural OT method described in Section~\ref{sec:map}.}
  \label{fig:EMNIST}
\end{figure}
In Figure \ref{fig:EMNIST}, we provide qualitative results of OTDD map from EMNIST (letter)~\citep{cohen2017emnist} dataset to all other *NIST dataset and USPS dataset. At this point, we can confirm three traits of OTDD map, which are mentioned at the end of \S\ref{sec:map}.

\begin{figure*}[ht!]
  \centering
  \begin{subfigure}{1\textwidth}
    \centering
    \includegraphics[width=1\linewidth]{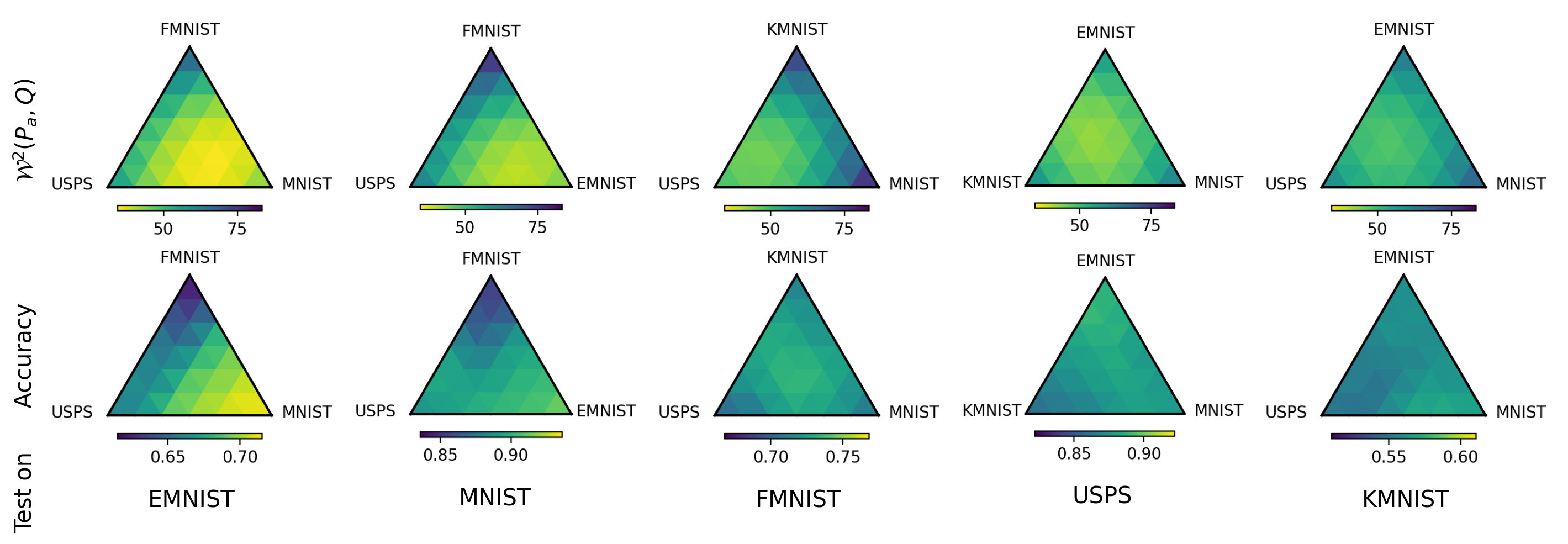}
  \end{subfigure}
  \caption{Relationship between the function $\cW^2(P_a, Q)$ and the accuracy of the fine-tuned model.
    The model trained on the projection dataset $\hat P_a$, i.e. the minimizer of $\cW^2(P_a, Q)$, tends to have a better generalization accuracy.
    The training datasets are marked on the vertexes of each ternary plot. Each ternary plot is an average of 5 runs with distinct random seeds.
  }
  \label{fig:ternary}
\end{figure*}

1) We don't assume a known source label to target label correspondence. So we can map between two irrelevant datasets such as EMNIST and FashionMNIST.
2) The map is invariant to the permutation of label assignment. For example, we show two different labelling in Figure 
\ref{fig:labelling}, 
and the final OTDD map will be the same.
3) It doesn't enforce the label-to-label mapping but would follow the feature similarity. From Figure \ref{fig:EMNIST}, we notice many cross-class mapping behaviors. For example, when the target domain is USPS~\citep{hull1994database} dataset, the lower-case letter "l" is always mapped to digit 1, and the capital letter "L" is mapped to other digits such as 6 or 0 because the map follows the feature similarity.

\subsection{Transfer learning on *NIST datasets }\label{sec:nist}

Next, we use our framework to generate new pretraining datasets for 
transfer
learning. 
Preceding works illustrate that the transfer learning performance can be quite sensitive to the type of test datasets 
if there is abundant training data from the test task~\citep[Table 1]{zhai2019large}. Thus, we will focus on the few-shot setting, where we only have few labeled data from the test task.
We first show that the generalization ability of training models has a strong correlation with the distance $\cW^2_{2,Q}(P_a, Q) $. Then we compare our framework with several baseline methods.

\paragraph{Setup}
Given $m$ labeled pretraining datasets $\{P_i\}$, we consider a few-shot task in which only a limited amount of data from the target domain is labeled, e.g. 5 samples per class. The goal is to find a single dataset of size comparable to any individual $P_i$ that yields the best generalization to the target domain when pre-training a model on it and fine-tuning on the target few-shot data. Here, we seek this training dataset within those generated by generalized geodesics $\{P_a\}$, which can be understood as weighted interpolations of the training datasets $\{P_i\}$. Note this includes individual datasets as particular cases when $a$ is a one-hot vector.




\begin{table*}[t]
  \caption{
  \textbf{Pretraining on synthetic data}. For each of the *NIST datasets, we treat it as the target domain and pretrain a neural net on a synthetic dataset generated as a combination of the remaining dataset with three interpolation methods. Here we show 5-shot transfer accuracy (mean $\pm$ s.d.~over 5 runs). The first baseline is to create a synthetic dataset as a training dataset by Mixup among datasets. For Mixup, we randomly sample data from each training dataset, and do the convex combination of them with weight $\hat a$ (see Eq. \eqref{eq:hat_a}). We use the same convex combination method in \S\ref{sec:comb}, thus this Mixup baseline is equivalent to our framework with suboptimal OTDD maps. The other two baselines (the bottom block) skip the transfer learning part, and directly train the model or solve 1-NN on the few-shot test dataset. 
  }
  \begin{center}
    \begin{small}
      \resizebox{\textwidth}{!}{
        \begin{tabular}{ccccccc}
          \toprule
          Methods                     & MNIST-M             & EMNIST              & MNIST               & FMNIST              & USPS                & KMNIST              \\
          \midrule
          OTDD barycentric projection & {\bf42.10$\pm$4.37} & {\bf67.06$\pm$2.55} & {\bf93.74$\pm$1.46} & {\bf70.12$\pm$3.02} & 86.01$\pm$1.50      & {\bf52.55$\pm$2.73} \\
          OTDD neural map             & 40.06$\pm$4.75      & 65.32$\pm$1.80      & 88.78$\pm$3.85      & 70.02$\pm$2.59      & 83.80$\pm$1.60      & 50.32$\pm$3.10      \\
          Mixup with weights $\hat a$                      & 33.85$\pm$2.22      & 60.95$\pm$1.38      & 88.68$\pm$1.57      & 66.74$\pm$3.79      & {\bf88.61$\pm$2.00} & 48.16$\pm$3.38      \\
          \midrule
          Train on few-shot dataset   & 19.10$\pm$3.57      & 53.60$\pm$1.18      & 72.80$\pm$3.10      & 60.50$\pm$3.07      & 80.73$\pm$2.07      & 41.67$\pm$2.11      \\
          1-NN  on few-shot dataset   & 20.95$\pm$1.39      & 39.70$\pm$0.57      & 64.50$\pm$3.32      & 60.92$\pm$2.42      & 73.64$\pm$2.35      & 40.18$\pm$3.09      \\
          \bottomrule
        \end{tabular}
      }
    \end{small}
  \end{center}
  \label{tab:compare}
\end{table*}

\paragraph{Connection to generalization}
The closed-form expression of $W_{2, \nu}^2 (\rho_a^G, \nu)$ (Prop.~\ref{prop:eq}) provides a distance between a base distribution $\nu$ and the distribution along generalized geodesic $\rho_a^G$ in Euclidean space.
We study its analog \eqref{eq:ds_eq} for labeled datasets $Q$ and $\{P_i\}$ and visualize it in Figure \ref{fig:ternary} (first row).
To investigate the generalization abilities of models trained on different datasets, we discretize the simplex $\Delta_2$ to obtain $36$ interpolation parameters $a$, and train a 5-layer LeNet  classifier on each $P_a$. Then we fine-tune all of these classifiers on the few-shot test dataset $Q$ with only 20 samples per each class. We control the same number of training iterations and fine-tuning iterations across all experiments.
The second row of Figure \ref{fig:ternary} shows fine-tuning accuracy. Comparing the first row and the second, we find the accuracy and $\cW^2(P_a, Q)$ are highly correlated. This implies that the model trained on the minimizer dataset of $\cW^2(P_a, Q)$ tends to have a better generalization ability. We fix the same colorbar range for all heatmaps across datasets to highlight the impact of training dataset choice. A more concrete visualization of the correlation between $\mathcal{W}^2(P_a, Q)$ and accuracy is shown in Figure 
\ref{fig:corr}.

For some test datasets, the choice of training dataset strongly affects the fine-tuning accuracy. For example, when $Q$ is EMNIST and the training dataset is FMNIST, the fine-tuning accuracy is only $\sim 60\%$, but this can be improved to $> 70\%$ by choosing an interpolated dataset closer to MNIST. This is reasonable because MNIST is more similar to EMNSIT than FMNIST or USPS. To some test datasets like FMNIST and KMNIST, this difference is not so obvious because all training datasets are all far away from the test dataset.


\paragraph{Comparison with baselines.}

Next, we compare our method with several baseline methods on NIST datasets. In each set of experiments, we select one *NIST dataset as the target domain, and use the rest for pre-training. We consider a 5-shot task, so we \textbf{randomly} choose 5 samples per class to be the labeled data, and treat the remaining samples as unlabeled. Our method first trains a model on $\widehat{P}_a$, and fine-tunes the model on the 5-shot target data. To obtain $\widehat{P}_a$, we use barycentric projection or neural map to approximate the OTDD maps from the test to training datasets. Our results are shown in the first two rows in Table \ref{tab:compare}.
Overall, transfer learning can bring additional knowledge from other domains and improve the test accuracy by at most 21$\%$. Among the methods in the first block, training on datasets generated by OTDD barycentric projection outperforms others except USPS dataset, where the difference is only about 2.6$\%$.
\subsection{Transfer learning on VTAB datasets}\label{sec:vtab}

Finally, we use our method for transfer learning with large-scale VTAB datasets~\citep{zhai2019large}. In particular, we take \href{https://www.robots.ox.ac.uk/~vgg/data/pets/}{Oxford-IIIT Pet dataset} as the target domain, and use \href{https://data.caltech.edu/records/mzrjq-6wc02}{Caltech101}, \href{https://www.robots.ox.ac.uk/~vgg/data/dtd/}{DTD}, and \href{https://www.robots.ox.ac.uk/~vgg/data/flowers/102/}{Flowers102} for pre-training. To encode a richer geometry in our interpolation, we embed the datasets using a masked auto encoder (MAE)~\citep{he2022masked} and learn the OTDD map in this ($\sim$200K dimensional) latent space.
Since OTDD barycentric projection consistently works better than OTDD neural map (see Table \ref{tab:compare}), we only use barycentric projection 
in this section.  We use ResNet-18 as the model architecture and pre-train the model on decoded MAE images (interpolated dataset) or original images (single dataset). 
\fan{Meanwhile, Mixup baseline is over pixel space and therefore
does not utilize embeddings at all.}

\begin{table}[ht]
\caption{\textbf{Transfer Learning on VTAB datasets}. The table shows relative improvement (w.r.t.~a no-transfer baseline) of test accuracy on \textsc{Oxford-IIIT Pet} (mean $\pm$ std over 5 runs) given only 1000 \fan{\text{randomly selected}} samples of this dataset to fine-tune. 
The first three rows show single-pretraining-dataset baselines, and the remaining rows show methods that pretrain on a synthetic interpolation of these three, generated using Mixup or our proposed OTDD Map, using uniform or $\hat a$ (see Eq.~\eqref{eq:hat_a}) dataset interpolation weights. 
The pooling baseline pretrains on a dataset including all the pre-training datasets.
To construct the sub-pooling pretraining dataset, for each training sample
 from the target dataset (\textsc{Pet}) we find its 10-nearest neighbors (in embedding space) from across all pretraining datasets, and label them as belonging to the class from the target domain.
}\label{tab:vtab}\centering
{\renewcommand{\arraystretch}{1.1}%
\begin{tabular}{cccc}
\toprule
Pre-Training & Map & Weights & Rel. Improv. ($\%$) \\ \midrule
\textsc{Caltech101}  & $-$ & $-$ & 59.68 $\pm$ 41.44 \\ 
\textsc{DTD} & $-$ & $-$ &-1.17 $\pm$ 9.52 \\ 
\textsc{Flowers102} & $-$ & $-$ & -2.45 $\pm$ 26.25 \\
Pooling & $-$ & $-$ & {28.96 $ \pm$ 18.29} \\
Sub-pooling & $-$ & $-$ & {3.00 $ \pm$ 19.10} \\
Interpolation & Mixup & uniform & 33.26 $\pm$ 21.30 \\ 
Interpolation & Mixup & $\hat a$ & 51.99  $\pm$ 34.10  \\ 
Interpolation & OTDD & uniform & 82.61 $\!\pm\!$ 25.93 \\ 
Interpolation & OTDD & $\hat a$ & \textbf{95.17$ \pm$ 20.57} \\
\bottomrule
\end{tabular}}
\end{table}

The pre-training interpolation dataset generated by our method has `optimal' mixture weights $a=(0.43,0.24, 0.33)$ for (\textsc{Caltech101}, \textsc{dtd}, \textsc{Flowers102}), suggesting a stronger similarity between the first of these and the target domain (\textsc{Pets}). This is consistent with the single-dataset transfer accuracies shown in Table \ref{tab:vtab}. However, their interpolation yields better transfer than any single dataset, particularly when using our full method (interpolating using OTDD map with optimal mixture weights). 

In Table \ref{tab:vtab}, we compute relative improvement per run, and then average these across runs; in other words, we compute the mean of ratios (MoR) rather than the ratio of means (RoM). Our reasoning for doing this was (i) controlling for the ‘hardness’ inherent to the randomly sampled subsets of \textsc{Pet} by relativizing before averaging and (ii) our observation that it is common practice to compute MoR when the denominator and numerator correspond to paired data (as is the case here), and the terms in the sum are sampled i.i.d. (again, satisfied in this case by the randomly sampled subsets of the target domain).

Table 2 shows a high deviation due to a particularly good result generated by the non-transfer learning baseline with seed 2, while other methods such as Caltech101 pretraining and Flowers102 pretraining had particularly bad results with the same seed.

\section{Conclusion and discussion}\label{sec:conclude}
The method we introduce in this work provides, as shown by our experimental results, a promising new approach to generate synthetic datasets as combinations of existing ones. Crucially, our method allows one to combine datasets even if their label sets are different, and is grounded on principled and well-understood concepts from optimal transport theory. Two key applications of this approach that we envision are: 
\begin{itemize}[leftmargin=*,noitemsep,topsep=0.5pt,parsep=0.5pt,partopsep=0.5pt]
  \item \textbf{Pretraining data enrichment}. Given a collection of classification datasets, generate additional interpolated datasets to increase diversity, with the aim of achieving better out-of-distribution generalization. This could be done even without knowledge of the specific target domain (as we do here) by selecting various datasets to play the role of the `reference' distribution.
  \item \textbf{On-demand optimized synthetic data generation}. Generate a synthetic dataset, by combining existing ones, that is `optimized' for transferring a model to a new (data-limited) target domain.
\end{itemize}

\paragraph{Complexity} 
The complexity of 
solving
OTDD \underline{barycentric projection} by Sinkhorn algorithm is $\cO(N^2 )$~\citep{dvurechensky2018computational}, where $N$ is the number of data in both datasets. This can be expensive for large-scale datasets. In practice, we solve the batched barycentric projection, i.e. take a batch from 
both
datasets and solve the projection from source 
to target batch, and we normally fix batch size $B$ as $10^4$. This reduces the complexity from $\cO(N^2 )$ to $\cO(BN)$.
The complexity of solving \underline{OTDD neural map} is $\cO(B K H)$, where $K$ is number of iterations, and $H$ is the size of the network. We always choose $K = \cO(N)$ in the experiments.
The complexity of solving all the \underline{$(2,Q)$-dataset distances} in \eqref{eq:ds_eq} is $\cO(m^2N)$ since we need to solve the dataset distance between each pair of training datasets.
\underline{Putting these pieces together}, the complexity of approximating the interpolation parameter $\hat{a}$ for the minimizer of \eqref{eq:ds_eq} is  $\cO(N(B + m^2 ))$. 

\paragraph{Memory}  As the number of pre-training tasks ($m$) increases, our method, which generates an interpolated label by concatenating labels from all tasks, creates an increasingly sparse vector. Consequently, the memory demands of the classifier's output layer, which is proportional to $m$, could rise significantly.

\paragraph{Barycentric projection vs Neural map}
These two versions of our method offer complementary advantages. While estimating the OT map allows for easy out-of-sample mapping and continuous generation, the barycentric projection approach often yields better downstream performance (Table \ref{tab:compare}). We hypothesize this is due to the barycentric projection relying on (re-weighted) \textit{real} data, while the neural map \textit{generates} data which might be noisy or imperfect.

\paragraph{Pixel space vs feature space} We present results with OTDD mapping in both pixel space (\S \ref{sec:nist}) and feature space (\S \ref{sec:vtab}). For the VTAB datasets with regular-sized images (e.g. $256\times 256\times 3$), we found that the feature space is more appropriate for measuring data distance. For small-scale images like NIST, feature space may be overkill because most foundation models are trained on images with a larger size. In our preliminary experiments with NIST datasets, we attempted a feature space approach using an off-the-shelf ResNet-18 model. However, we encountered challenges in achieving convergence when training OTDD neural maps with PyTorch ResNet-18 features.

\paragraph{High variance issue}
Our method is not limited to the data scarcity regime, but indeed this is the most interesting one from the transfer learning perspective, which is why we assume limited labeled data (but potentially much more unlabeled data) from the target domain distribution. This is a typical few-shot learning scenario.
The quality of a learned OT map will likely depend on the number of samples used to fit it, and might suffer from high variance. To mitigate this in our setting, we opt for augmenting our 
 dataset by generating additional pseudo-labeled data via kNN (Fig. \ref{fig:convex_comb}). Recall that we do have access to more unlabeled data from the target domain, which is a common situation in practice.

\paragraph{Limitations}
Our method for generating a synthetic dataset relies on solving OTDD maps from the test dataset to each training dataset.
These OTDD maps are tailored to the considered test dataset and
can not be reused for a new test dataset. Another limitation is our framework is based on model training and fine-tuning pipeline. This can be resource-demanding for large-scale models, like GPT~\citep{brown2020language} or other similar models.
\fan{Finally, if at least one of the datasets is imbalanced, our OTDD map will struggle to match the class with similar marginal distributions.}



\begin{acknowledgements} 
We thank Yongxin Chen and Nicolò Fusi for their invaluable comments, ideas, and feedback.
We extend our gratitude to the anonymous reviewers for their useful feedback that significantly improved this work.
\end{acknowledgements}

\bibliography{fan_236}

\newpage
\appendix
\onecolumn

\section{Proofs}\label{sec:proof}
\begin{proof}[Proof of Lemma \ref{lem:convex_w2}]
  By \citet[\S4.4]{santambrogio2017euclidean}, the result holds when $m=2$. Then Proposition 7.5 in \citet{agueh2011barycenters} extends the result to the case of $m>2$.
\end{proof}

\begin{proof}[Proof of Proposition \ref{prop:eq}]
  Since linear combination preserves cyclically monotonicity, $\sum_{i=1}^m a_i T_i^*(x)$ is the optimal map from $\nu$ to $\rho_a^G$~\citep{mccann1995existence}.
  Then according to the definition of $W_{2, \nu }(\cdot , \cdot ) $, we can write
  \begin{align}\label{eq:pf}
    W_{2, \nu }^2(\rho^G_a, \nu )
    =
    \int \left \|x - \sum_{i=1}^m a_i T_i^*(x) \right\|^2 \d \nu (x).
  \end{align}
  For scalars $p , q_1,\ldots, q_m$, it holds that
  \begin{align}
    \left(p - \sum_{i=1}^m a_i q_i \right)^2 & = p^2 + \sum_{i=1}^m a_i^2 q_i^2 - 2  \sum_{i=1}^m a_i p q_i  + \sum_{i\ne j} a_i a_j q_i q_j                           \\
                                             & = p^2 + \sum_{i=1}^m (a_i - a_i \sum_{j \ne i} a_j ) q_i^2 - 2  \sum_{i=1}^m a_i p q_i  + \sum_{i\ne j} a_i a_j q_i q_j \\
                                             & =  \sum_{i=1}^m a_i (p-q_i)^2 - \frac{1}{2} \sum_{i \ne j} a_i a_j (q_i - q_j)^2.
  \end{align}
  Plugging this equality into \eqref{eq:pf} gives
  \begin{align}
    W_{2, \nu }^2(\rho^G_a, \nu )
     & =
    \int \left( \sum_{i=1}^m a_i   \|x -  T_i^*(x)\|^2 - \frac{1}{2} \sum_{i \ne j} a_i a_j \| T_i^*(x) -  T_j^*(x)\|^2 \right) \d \nu (x)         \\
     & =   \sum_{i=1}^m a_i \int  \|x -  T_i^*(x)\|^2 \d \nu (x) - \frac{1}{2} \sum_{i \ne j} a_i a_j \int \| T_i^*(x) -  T_j^*(x)\|^2  \d \nu (x) \\
     & = \sum_{i=1}^m a_i W_{2,\nu }^2(\mu_i, \nu ) - \frac{1}{2} \sum_{i \neq j} a_i a_j W_{2,\nu }^2(\mu_i, \mu_j ).
  \end{align}
\end{proof}
\begin{proof}[Proof of Proposition \ref{prop:metric}]
  Firstly, $\cW_{2,Q}$ is symmetric and nonnegative by definition. It is non-degenerate since $\cW_{2,Q} (P_i,P_j) \ge d_\OT (P_i, P_j) $ and $d_\OT$ is a metric. Finally, we show it satisfies the triangular inequality. Indeed,
  \begin{align}
     & ~~~~~~\cW_{2,Q}(P_1, P_3)                                                                      \\
     & =\left( \int  \| x_1 -  x_3 \|^2 +
    W_2^2(\alpha_{y_1}, \alpha_{y_3})
    \d Q(z) \right)^{1/2}                                                                             \\
     & \le \left( \int  (\| x_1 -  x_2 \| + \| x_2 -  x_3 \| )^2 +
    (W_2(\alpha_{y_1}, \alpha_{y_2})
    + W_2(\alpha_{y_2}, \alpha_{y_3})
    )^2
    \d Q(z) \right)^{1/2}                                                                             \\
     & \le \left( \int  \| x_1 -  x_2 \|^2 + W^2_2(\alpha_{y_1}, \alpha_{y_2}) \d Q(z)  \right)^{1/2}
    + \left( \int \| x_2 -  x_3 \|^2 + W^2_2(\alpha_{y_2}, \alpha_{y_3})
    \d Q(z) \right)^{1/2}                                                                             \\
     & = \cW_{2,Q}(P_1, P_2) +  \cW_{2,Q}(P_2, P_3),
  \end{align}
  where the first inequality is the triangular inequality and the second inequality is the Minkowski inequality.
\end{proof}
\section{Implementation details of OTDD map}\label{sec:neural_map}

\paragraph{OTDD barycentric projection}
We use the implementation \url{https://github.com/microsoft/otdd} to solve OTDD coupling. The rest part is straightforward.
\paragraph{OTDD neural map}
To solve the problem \eqref{eq:max-min}, we parameterize $f, G, \ell$ to be three neural networks. In NIST dataset experiments, we parameterize $f$ as ResNet~\footnote{\url{https://github.com/harryliew/WGAN-QC}} from WGAN-QC~\citep{liu2019wasserstein}, and take feature map $G$ to be UNet\footnote{\url{https://github.com/milesial/Pytorch-UNet}}~\citep{ronneberger2015u}. We generate the labels $\bar y$ with a pre-trained classifier $\ell(\cdot)$, and use a LeNet or VGG-5 with Spinal layers\footnote{\url{https://github.com/dipuk0506/SpinalNet}} \citep{kabir2022spinalnet} to parameterize $\ell(\cdot)$.  In 2D Gaussian mixture experiments, we use Residual MLP to represent all of them.

We remove the discriminator's condition on label to simplify the loss function as
\begin{align}
  \sup_f \inf_G  \int \bigl( \underbrace{\| x- G( z)\|_2^2}_\text{feature loss}   + \underbrace{W_2^2(\alpha_{y}, \alpha_{\bar y} )}_\text{label loss} \bigr) \d Q(z) \underbrace{- \int f(\bar x ) \d Q(z)
    + \int f(x') \d P(z') }_\text{discriminator loss}.
\end{align}
In this formula, we assume both $y$ and $\bar y$ are hard labels, but in practice, the output of $\ell(\cdot )$ is a soft label. Simply taking the \texttt{argmax} to get a hard label can
break the computational graph,
so we replace the label loss $W_2^2(\alpha_{y}, \alpha_{\bar y} )$ by $y^\top M \bar{y}$, where $y$ is the one-hot label from dataset $Q$.
And $M \in \mR_{\ge 0}^{C_{Q} \times C_{P}}$ is the label-to-label matrix where $M(i,j) := W_2^2(\alpha_{y_i}, \alpha_{y_j}).$
 The matrix $M$ is precomputed before the training, and is frozen during the training.

We pre-train the feature map $G$ to be an identity map before the main adversarial training. We use the Exponential Moving Average\footnote{\url{https://github.com/fadel/pytorch_ema}} of the trained feature maps as the final feature map.

\paragraph{Data processing} For all the *NIST datasets, we rescale the images to size $32\times 32$, and repeat their channel 3 times and obtain 3-channel images. We use the default train-test split from \texttt{torchvision}.
For the VTAB datasets, we use a masked auto-encoder with 
196 batches and 1024 embed dimension based on ViT-Large. So the final embedding dimension is $197 \times 1024 = 201728$.  We also use the default train-test split from \texttt{torchvision}.

\paragraph{Hyperparameters} For the experimental results in \S\ref{sec:nist}, we use the OTDD neural map and train them using Adam optimizer with learning rate $10^{-3}$ and batch size 64. We train a LeNet for 2000 iterations, and fine-tune for 100 epochs. Regarding the comparison with other baselines in \S\ref{sec:nist}, for transfer learning methods, we train a SpinalNet for $10^4$ iterations, and fine-tune it for $2000$ iterations on the test dataset. Training from scratch on the test dataset takes also 2000 iterations. For the results in \S\ref{sec:vtab}, we pre-train the ResNet-18 model for 5 epochs, then fine-tune the model on the few-shot dataset for 10 epochs. During fine-tuning, we still let the whole network tunable. The batch size is 128, and the learning rate is $10^{-3}$.

\section{Discussions over complexity-accuracy trade-off}

We agree that our method is more computationally demanding than Mixup in general.
Specifically, we consider Mixup and our methods to occupy different points of a compute-accuracy trade-off characterized by the expressivity of the geodesics between datasets they define. That being said, the trade-off is nevertheless not a prohibitive one, as shown by the fact that we can scale our method to VTAB-sized datasets with a very standard GPU setup.

‘Vanilla’ mixup with uniform dataset weights is indeed quite cheap (but, as shown in Table \ref{tab:vtab}, considerably worse than alternatives). On the other hand, the version of Mixup that uses the ‘optimal’ mixture weights (labeled Mixup - optimal in Table \ref{tab:vtab}, and the only Mixup version in Table 1) requires solving Eq. \eqref{eq:ds_eq}, which involves non-trivial computing to obtain OTDD maps. In the context of the trade-off spectrum described above, Mixup with optimal weights is strictly in between vanilla Mixup and OTDD interpolation.

\section{Additional results}\label{sec:nist_maps}

\begin{figure}[h]
  \centering
  \begin{subfigure}{0.35\textwidth}
    \includegraphics[width=1\linewidth]{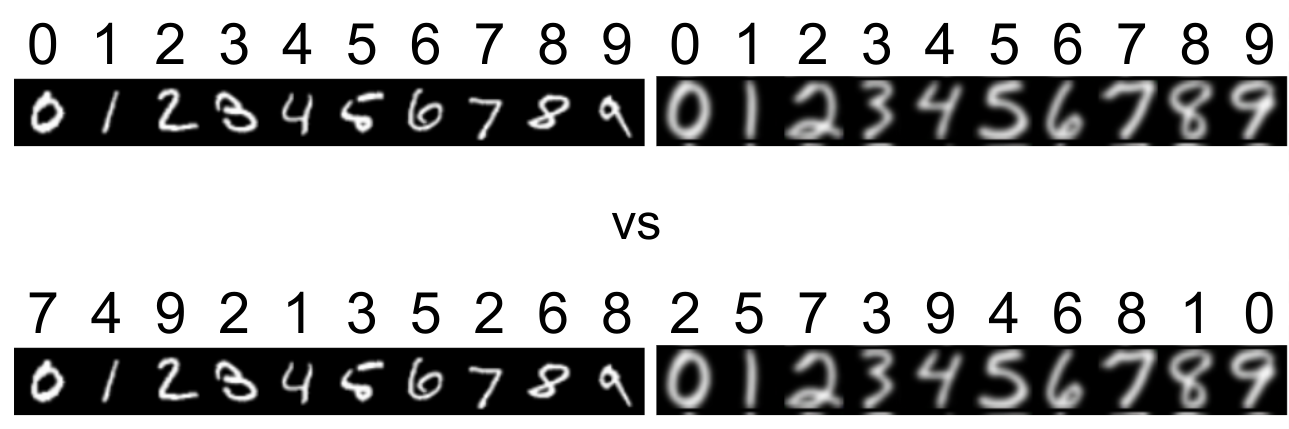}
  \end{subfigure}
  \caption{The numbers above images are the labels. In the first labelling method, all 0 MNIST digits are assigned as class "0", and they are labelled as class "7" in the bottom labelling.}
  \label{fig:labelling}
\end{figure}
\subsection{OTDD neural map visualization}

We show the OTDD neural map between  2D Gaussian mixture models with 16 components in Figure \ref{fig:chess}. This example is very special so that we have the closed-form solution of OTDD map. The feature map is a identity map and the pushforward label is equal to the corresponding class that has the same conditional distribution $p(x|y)$ as source label. For example, the sample from top left corner cluster is still mapped to the top left corner cluster, and the label is changed from blue to orange. This map achieves zero transport cost. Since the transport cost is always non-negative, this map is the optimal OTDD map.
However, \cite{asadulaev2022neural,bunne2022supervised} enforce mapping to preserve the labels, so with their methods, the blue cluster would still map to the blue cluster. Thus their feature map is highly non-convex and more difficult to learn. We refer to Figure 5 in \citet{asadulaev2022neural} for their performance on the same example. Compared with them, our pushforward dataset aligns with the target dataset better.

\begin{figure}[h!]
  \centering
  \begin{subfigure}{1\textwidth}
    \centering
    \includegraphics[width=0.85\linewidth]{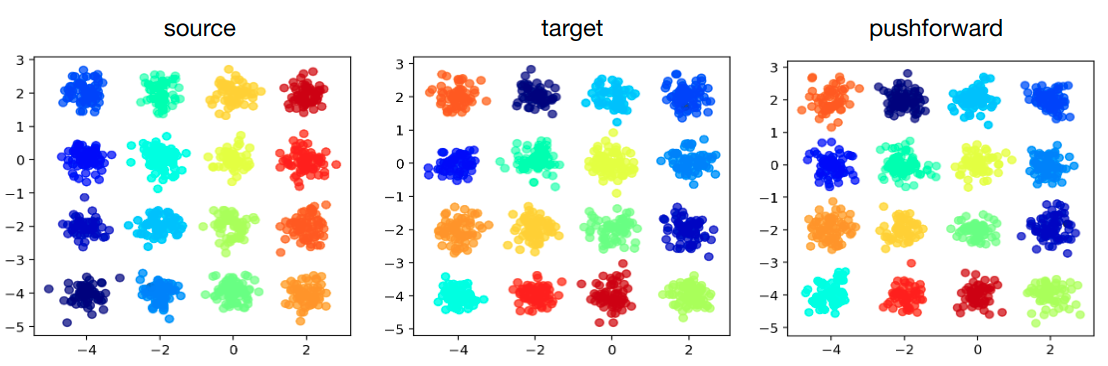}
  \end{subfigure}
  \caption{OTDD neural map for 2D Gaussian mixture distributions.}
  \label{fig:chess}
\end{figure}

\subsection{McCann's interpolation between datasets}

Our OTDD map can be extended to generate  McCann's interpolation between datasets. We propose an anolog of McCann's interpolation \eqref{eq:mccan} in the dataset space. We define McCann's interpolation between datasets $P_0$ and $P_1$ as
\begin{align}
  P^M_t: = ((1-t) {\rm Id} + t \cT^* )\sharp P_0 , \quad t \in [0,1],
\end{align}
where $\cT^*$ is the optimal OTDD map from $P_0$ to $P_1$ and $t$ is the interpolation parameter. The superscript $M$ of $P_t^M$ means McCann. We use the same convex combination method in \S\ref{sec:comb} to obtain samples from $P^M_t$.
Assume $(x_0,y_0) \sim P_0,~ (x_1,y_1) = \cT^*(x_0,y_0)$ and $P_0, P_1$ contain 7, 3 classes respectively, i.e. $y_0 \in \{0,1\}^7, y_1 \in \{0,1\}^3$. Then the combination of features is $x_t = (1-t) x_0 + t x_1 $, and the combination of labels is
\begin{align}
  y_t = (1-t)
  \begin{bmatrix}
    y_0 \\ \mathbf{0}_{3}
  \end{bmatrix}
  + t
  \begin{bmatrix}
    \mathbf{0}_7 \\ y_1
  \end{bmatrix}.
\end{align}
Thus $(x_t, y_t)$ is a sample from $((1-t) {\rm Id} + t \cT^* )\sharp P_0 $. We visualize McCann's interpolation between two Gaussian mixture distributions in Figure \ref{fig:mccan_2d}. This method can map the labeled data from one dataset to another, and do the interpolation between them. Thus we can use it to map abundant data from an external dataset, to a scarce dataset for data augmentation. For example, in Figure \ref{fig:mccan_2d_imb}, the target dataset only has 30 samples, but the source dataset has 60000 samples. We learn the OTDD neural map between them and solve their interpolation. We find that $P_1^M$ creates new data out of the domain of the original target distribution, which Mixup~\citep{zhang2018mixup} can not achieve. Thus, the data from $P_t^M$ for $t$ close to 1.0 can enrich the target dataset, and be potentially used in data augmentation for classification tasks.

\begin{figure}[h!]
  \centering
  \begin{subfigure}{1\textwidth}
    \centering
    \includegraphics[width=1\linewidth]{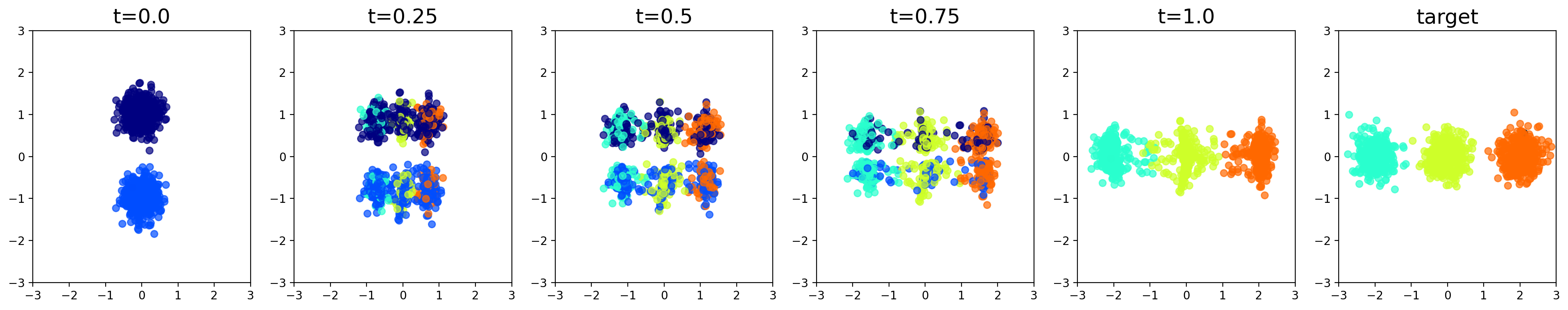}
  \end{subfigure}
  \caption{McCann's interpolation for 2D labelled datasets. Each color represents a class. When $t \rightarrow 1.0$, the samples within blue classes become less and less, and finally disappear when $t=1.0$.}
  \label{fig:mccan_2d}
\end{figure}

\begin{figure}[ht!]
  \centering
  \begin{subfigure}{1\textwidth}
    \centering
    \includegraphics[width=1\linewidth]{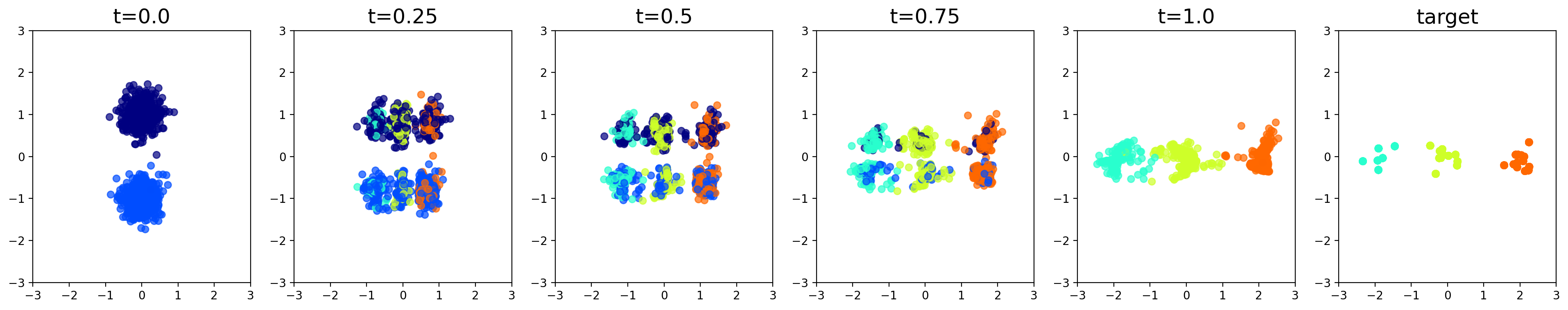}
  \end{subfigure}
  \caption{Data augmentation by mapping  an external dataset to a few-shot dataset.
  }
  \label{fig:mccan_2d_imb}
\end{figure}

\subsection{Correlation study of *NIST experiments}
 A more concrete visualization of the correlation  between $\mathcal{W}^2(P_a, Q)$ and *NIST transfer learning test accuracy is shown in Figure \ref{fig:corr}. Among all datasets, USPS and KMNIST lack correlation. 
 We believe it’s caused by (i) small variance in the distances from pretraining dataset to target dataset, implying a limited relative diversity of datasets on which to draw on and (ii) (in the case of USPS) a very simple task where baseline accuracy is already very high and hard to improve upon via transfer. 
 
\begin{figure*}[ht!]
  \centering
  \begin{subfigure}{1\textwidth}
    \centering
    \includegraphics[width=1\linewidth]{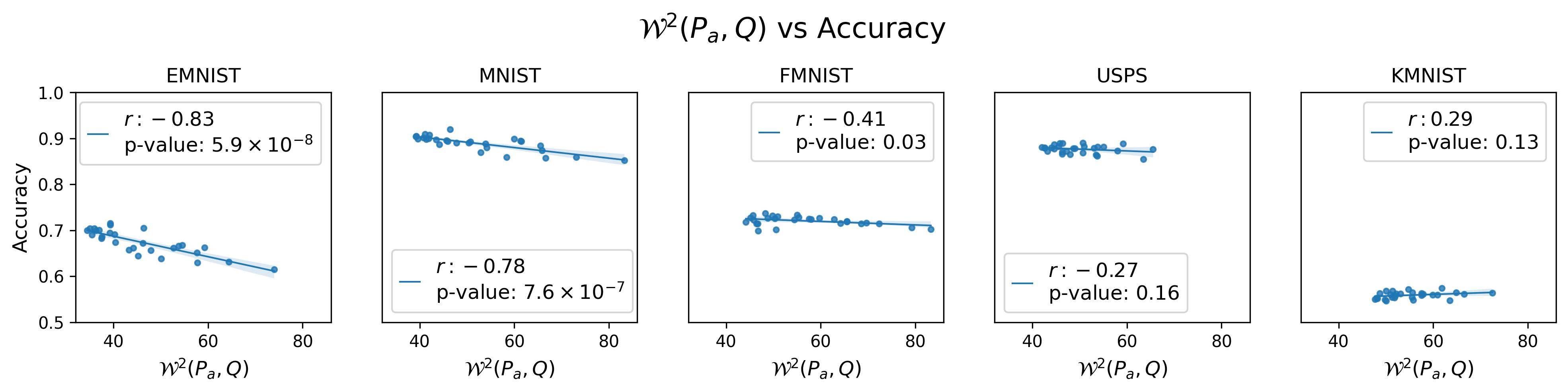}
  \end{subfigure}
  \caption{Pearson correlation between the (averaged) function $\cW^2(P_a, Q)$ and the test accuracy of the fine-tuned model.
    Most datasets present a negative correlation between $\cW^2(P_a, Q)$ and the accuracy.
    When test dataset is USPS or KMNIST (rightmost two), 
    all three training datasets are similarly distant to the test dataset;
    thus, the range of $\cW^2(P_a, Q)$ is not wide enough to show an obvious negative correlation.
    This explains the nearly zero slope and relatively large $p$-value for those two datasets.
    Similar pattern has been observed in \citet[Figure 5(a)]{yeaton2022hierarchical}.
  }
  \label{fig:corr}
\end{figure*}

\subsection{Fine-grained analysis over $\cW^2 (P_a,W)$ in *NIST experiments}

In Table \ref{tab:w_stat}, we provide a more fine-grained analysis for different aspects of $\mathcal{W}(P_a,Q)$ and their effect on transfer accuracy. To do so, we provide the min, median, range, and standard deviation of $\mathcal{W}(P_a,Q)$ in the table below. In addition, as a proxy for the hardness / best possible gain from transfer learning, we show in the last column \textit{OTDD accuracy} minus \textit{few shot accuracy}, where \textit{OTDD accuracy} and \textit{few shot accuracy} are the mean accuracies in Rows 1 and 4, respectively, in Table \ref{tab:compare}. 

Based on these statistics, we make the following observations on the relation between $\mathcal{W}(P_a,Q)$ and transfer accuracy:

\begin{itemize}
\item The accuracy improvement is strongly driven by $\min_a \mathcal{W}(P_a,Q)$. EMNIST and MNIST are with relatively smaller $\min_a \mathcal{W}(P_a,Q)$ and share the largest improvement margin. On the other hand, FMNIST and KMNIST as $Q$ have the largest $\mathcal{W}(P_a,Q)$ to the other pre-training datasets, and have relatively smaller accuracy gain. In other words, the correlation between distance and accuracy is stronger in the part of the convex dataset polytope that is closest to the target dataset.
\item The strength of the correlation between $\mathcal{W}(P_a,Q)$ and accuracy seems to depend on the \textbf{range} and \textbf{standard deviation} of he former. On the one hand, \textbf{settings with low dynamic range in $\mathcal{W}(P_a,Q)$ (like USPS and EMNIST) make it harder to observe meaningful differences in accuracy}. On the other hand, this indicates that those datasets are roughly (or at least \textbf{more}) equidistant from all pretraining datasets, and therefore any convex combination of them will also be close to equidistant from the target, yielding no visible improvement.
\item Intrinsic task hardness matters. Consider USPS: all pretraining datasets, regardless of distance, seem to yield very similar accuracy on it, and it has the lowest accuracy gain (only $\sim$5\%) among 5 tasks. But considering that the no-transfer (i.e. 5-shot) accuracy is already almost 81\%, it is clear that the benefit from transfer learning is “a priori” limited, and therefore all pretraining datasets yield a similar minor improvement.
\end{itemize}

\begin{table}[h]
\centering
\caption{Statistics of $\cW(P_a,Q)$
 and transfer accuracy in *NIST experiments (\S \ref{sec:nist}).}
\begin{tabular}{|l|c|c|c|c|c|}
\hline
Test dataset & 
\begin{tabular}[c]{@{}c@{}} Mean of \\ $\mathcal{W}(P_a,Q)$ \end{tabular}
& \begin{tabular}[c]{@{}c@{}} Median of \\
$\mathcal{W}(P_a,Q)$ \end{tabular} & \begin{tabular}[c]{@{}c@{}} Range of \\
$\mathcal{W}(P_a,Q)$ \end{tabular}  & \begin{tabular}[c]{@{}c@{}} Standard deviation \\
 of $\mathcal{W}(P_a,Q)$ \end{tabular} & \begin{tabular}[c]{@{}c@{}} Mean of accuracy \\
 improvement \end{tabular}  \\
\hline
EMNIST & 34.41 & 43.71 & 39.58 & 9.94 & 13.46 \\
MNIST & 39.13 & 49.04 & 44.17 & 11.35 & 20.94 \\
FMNIST & 44.19 & 54.75 & 39.11 & 10.64 & 10.62 \\
USPS & 42.04 & 48.32 & \textbf{23.49} & \textbf{6.13} & 5.28 \\
KMNIST & 47.65 & 53.92 & \textbf{24.83} & \textbf{6.19} & 10.88 \\
\hline
\end{tabular}
\label{tab:w_stat}
\end{table}

\subsection{Full results of VTAB experiments}
In Section \ref{sec:vtab}, we only showed the relative improvement of the test accuracy compared to non-pretraining. Here we will show the full test accuracy results. We keep the hyper-parameters consistent through all pre-training datasets. Table \ref{tab:vtab_full} clearly shows that the interpolation dataset with optimal weight assigned by our method can have a better performance than a na\"{\i}ve uniform weight. And with the same weight, our OTDD map will give a higher accuracy than Mixup because Mixup does not use the information from the reference dataset (see Figure \ref{fig:proj_2d}).

\paragraph{Poor sub-pooling performance} We show the sub-pooling baseline as a non-trivial method to combine datasets. However, it performs poorly, and we believe there are two main reasons for this. First, this baseline wastes relevant label data, by discarding the original labels of the pretraining dataset and replacing them with the inputted nearest-neighbor label from the target examples. Secondly, it only uses the neighbors of the pet dataset, leaving all other datapoints unused.

\begin{table}[H]
\caption{Test accuracy (mean $\pm$ std over 5 runs in percent) of 1000-shot learning on Oxford-IIIT Pet test dataset. 
Non-transfer learning skips the pre-training step.
}\label{tab:vtab_full}
\centering
{\renewcommand{\arraystretch}{1.1}%
\begin{tabular}{|cc|c|}
\hline
\multicolumn{1}{|c|}{\multirow{9}{*}{Transfer learning}} & OTDD map (optimal weight) 
  & \textbf{22.60 $\pm$ 1.01} \\ \cline{2-3} 
\multicolumn{1}{|c|}{}                    & OTDD map (uniform weight)           & 21.06 $\pm$ 0.45 \\ \cline{2-3} 
\multicolumn{1}{|c|}{}                    &   Mixup (optimal weight)      & 17.45  $\pm$ 2.2 \\ \cline{2-3} 
\multicolumn{1}{|c|}{}                    & Mixup (uniform weight)         & 15.4 $\pm$ 1.56 \\ \cline{2-3} 
\multicolumn{1}{|c|}{}                    & \textsc{Caltech101}  & 18.24 $\pm$ 3.42 \\ \cline{2-3} 
\multicolumn{1}{|c|}{}                    &  \textsc{DTD}   &11.46 $\pm$ 0.68 \\ \cline{2-3} 
\multicolumn{1}{|c|}{}                    & \textsc{Flowers102}    & 11.11 
 $\pm$ 1.92 \\  \cline{2-3} 
 \multicolumn{1}{|c|}{}      & \textsc{Pooling}    & 14.88 $\pm$
 0.57 \\  \cline{2-3} 
 \multicolumn{1}{|c|}{}       & \textsc{Sub-pooling}    & 14.88 $\pm$
 0.57 \\ \hline 
\multicolumn{2}{|c|}{Non-transfer learning}                           & 11.71 $\pm$ 1.65 \\ \hline
\end{tabular}
}
\end{table}

\end{document}